\documentclass{article}

\usepackage{microtype}
\usepackage{subfigure}
\usepackage{booktabs} 

\usepackage{amsmath}
\usepackage{amssymb}
\usepackage{amsfonts}
\usepackage{amsthm}
\usepackage[dvipdfmx]{graphicx}
\usepackage{mathtools}
\usepackage{algorithmic}
\usepackage{algorithm}
\usepackage{txfonts}

 \usepackage[dvipdfmx]{graphicx}
 \newtheorem{thm}{Theorem}
\newtheorem{lem}[thm]{Lemma}

\newtheorem{ass}{Assumption}

\usepackage{mathrsfs}

\usepackage{hyperref}

\usepackage[accepted]{icml2018}

\icmltitlerunning{Analysis of Minimax Error Rate for Crowdsourcing and Its Application to Worker Clustering Model}

\begin{document}

\twocolumn[
\icmltitle{Analysis of Minimax Error Rate for Crowdsourcing\\
and Its Application to Worker Clustering Model}


\icmlsetsymbol{equal}{*}

\begin{icmlauthorlist}
\icmlauthor{Hideaki Imamura}{ut,riken}
\icmlauthor{Issei Sato}{ut,riken}
\icmlauthor{Masashi Sugiyama}{riken,ut}
\end{icmlauthorlist}

\icmlaffiliation{ut}{The University of Tokyo, Tokyo, Japan}
\icmlaffiliation{riken}{RIKEN, Tokyo, Japan}

\icmlcorrespondingauthor{Hideaki Imamura}{imamura@ms.k.u-tokyo.ac.jp}

\icmlkeywords{Machine Learning, ICML, Crowdsourcing, Minimax Lower Bound}

\vskip 0.3in
]


\printAffiliationsAndNotice{The code for the experiment can be used in https://github.com/HideakiImamura/MinimaxErrorRate.\\}  

\begin{abstract}
	While crowdsourcing has become an important means to label data, there is great interest in estimating the ground truth from unreliable labels produced by crowdworkers. 
	The Dawid and Skene (DS) model is one of the most well-known models in the study of crowdsourcing.
	Despite its practical popularity, theoretical error analysis for the DS model has been conducted only under restrictive assumptions on class priors, confusion matrices, or the number of labels each worker provides.
	In this paper, we derive a minimax error rate under more practical setting for a broader class of crowdsourcing models including the DS model as a special case.
	We further propose the worker clustering model, which is more practical than the DS model under real crowdsourcing settings.
	The wide applicability of our theoretical analysis allows us to immediately investigate the behavior of this proposed model, which can not be analyzed by existing studies.
	Experimental results showed that there is a strong similarity between the lower bound of the minimax error rate derived by our theoretical analysis and the empirical error of the estimated value. 
\end{abstract}

\section{Introduction}
	Crowdsourcing has become an essential tool for large-scale data collection in machine learning.
	While crowdsourcing provides a less expensive means of labeling data, the data annotated by a crowd can be of low quality because crowd workers are often non-experts and are sometimes even adversarial.
	Many crowdsourcing services try to solve this problem by providing redundancy for labeling, i.e., by collecting multiple labels from different workers for each task \cite{Ipeirotis_et_al_2010, Welinder_et_al_2010, Snow_et_al_2008}.
	This strategy raises the question: how can the ground truths be estimated from noisy and redundant labels?	
	Estimating the ground truth is difficult in the {\it laissez-faire crowdsourcing setting} in which a large number of workers in the world can freely label as many tasks as they want.
	In this setting, a small number of workers annotate data a large number of times, while most workers perform annotation only a few times.
	That is,  the number of tasks to be labeled per worker typically follows Zipf's law. 
	In this paper, we focus on this setting because it is more realistic than the {\it non-laissez-faire crowdsourcing setting}, in which the number of tasks to be labeled per worker is almost constant across all workers and which is assumed in the experiments of many previous works \cite{Welinder_et_al_2010, Snow_et_al_2008}.

	In the context of estimating the ground truth of each task from noisy and redundant labels, \citet{Dawid_Skene_1979} conducted pioneering research.
	It is assumed in the Dawid and Skene (DS) model that each worker has his or her own confusion matrix.
	The ground truth of each task and the confusion matrix of each worker are jointly estimated by the Expectation-Maximization (EM) algorithm \cite{Dempster_et_al_1977}.	
	
	Although the DS model has had empirical success \cite{Welinder_et_al_2010, Snow_et_al_2008}, there are two major problems.
	First, there are few theoretical error analyses on the performance conducted for the DS model, and the existing theoretical analyses are only valid under strong assumptions.
	For example, \citet{Gao_et_al_2016} assumed that the class prior is uniform and  \citet{Zhang_et_al_2014} assumed that the entries of the confusion matrices are strictly positive.
	Second, the experimental validity of the existing methods is confirmed only when the number of tasks to be labeled per worker is almost constant on both synthetic and real-world data \cite{Welinder_et_al_2010, Snow_et_al_2008}.
	
	To alleviate the first problem, we provide a novel theoretical error analysis under milder assumptions based on Fano's method, which is a useful minimax lower-bounding technique \citep{Yu_1997}.
	 Our theoretical analysis is applicable to all models that use the ground truth of each task and the confusion matrix of each worker, including the DS model and its extensions, thanks to the fact that our assumptions are much weaker than those required in previous work.
	 
	To alleviate the second problem, we extend the DS model to be able to handle worker clusters.
	Intuitively, by clustering workers, even when the number of tasks to be labeled per worker is small, the number of tasks to be labeled per \emph{worker cluster} can be increased and thus estimation can be stabilized.
	Note that our widely applicable theoretical analysis explained above allows us to investigate the behavior of the proposed clustering model, while the existing theoretical analysis methods cannot be used due to their restrictive assumptions.
	We experimentally show the usefulness of our worker clustering (WC) model for reducing the influence of variations in the number of tasks to be labeled per worker.
	We also numerically confirm the validity of our theoretical error analysis.
	 
\section{Related Work}
		\begin{table*}[ht]
			\caption{Relationship between existing work and this work.
			DS stands for Dawid and Skene, while WC stands for worker clustering.}
				\label{ta: rw}
				\vspace*{-1mm}
				\begin{center}
				\begin{small}
				\begin{sc}
					\scalebox{1.0}[1.0]{
					\begin{tabular}{l||c | c |c | c}
						\toprule
						  & Model & Theoretical analysis & laissez-faire setting & Class prior   \\
						\midrule
						\citet{Dawid_Skene_1979} & DS & - & - & Not uniform \\ \hline
						\citet{Karger_et_al_2011} & Other & \checkmark & - & No class prior \\ \hline
						\citet{Liu_Wang_2012} & DS & - & - & Not uniform \\ \hline 
						\citet{Gao_et_al_2016} & DS & \checkmark & - & Uniform \\ \hline 
						This work & DS \& WC & \checkmark & \checkmark  & Not uniform \\ \hline
						\bottomrule
					\end{tabular}
					}
				\end{sc}
				\end{small}
				\end{center}
				\vskip -0.1in
		\end{table*}
		
	A large number of studies on the quality assurance of data collected by crowdsourcing have been conducted 
	\cite{Karger_et_al_2011,
		Bachrach_et_al_2012,
		Zhou_et_al_2012,
		Chen_et_al_2013,
		Karger_et_al_2013,
		Parisi_et_al_2014,
		Venanzi_et_al_2014,
		Karger_et_al_2014,
		Tian_et_al_2015}.	
     	One of the most practical and pioneering studies in this field is the Dawid and Skene (DS) model \cite{Dawid_Skene_1979}.
	The DS model is based on an estimation paradigm that uses the ground truths of all tasks and the confusion matrices of all workers; their inference algorithm is based on the EM algorithm
	\cite{Dempster_et_al_1977}.
	Application research using their method has been actively carried out \cite{Hui_Walter_1980, Smyth_et_al_1995, Albert_Dodd_2004}.
	As many experiments on synthetic and real-world data have demonstrated, the DS model is practical for estimating the ground truth of each task from noisy and redundant labels.
	
	There is also many studies that uses the EM algorithm
	\cite{Whitehill_et_al_2009,
		Welinder_et_al_2010,
		Rayker_et_al_2010, 
		Welinder_Perona_2010,
		Liu_et_al_2012,
		Liu_Wang_2012,
		Zhang_et_al_2014}.
	\citet{Liu_Wang_2012} assume priors over the class prior of the ground truths of tasks and the confusion matrices of workers.
	\citet{Zhang_et_al_2014} devised an effective way to initialize the EM algorithm.
	Specifically, the initial values of the confusion matrices of all workers were estimated using the method of moments.
	\citet{Karger_et_al_2011} proposed an iterative algorithm for binary labeling problems which is not based on the DS model, and gave a theoretical analysis under strong assumptions.
	
	The idea of clustering worker was already proposed by \citet{Venanzi_et_al_2015} and \citet{Moreno_et_al_2015}. 
	The problem setting of \citet{Venanzi_et_al_2015} supposed that the inputs of the model are not labels but real-valued vectors, which is different from our problem setting.
	\citet{Moreno_et_al_2015} solved the same problem as ours, but assumed a complex generation process for confusion matrices and labels given by workers. 
	To give the worker clustering structure, their model becomes complex, and difficult to analyze theoretically. 

	As a number of algorithms have been proposed, statistical understanding of crowdsourcing has been actively researched, such as
	\cite{Ghosh_et_al_2011,
	Dalvi_et_al_2013,
	Karger_et_al_2014,
	Zhang_et_al_2014,	
	Gao_et_al_2016,
	Bonald_Combes_2017}.
	These studies are divided into two types.
	One is a theoretical analysis of the estimation paradigm that uses the ground truths and confusion matrices
	\cite{Zhang_et_al_2014,
	Gao_et_al_2016}.
	The other is a theoretical analysis of their own models and algorithms
	\cite{
	Ghosh_et_al_2011,
	Dalvi_et_al_2013,
	Karger_et_al_2014,
	Bonald_Combes_2017}.
	Our theoretical analysis is of the former category.
	
	\citet{Zhang_et_al_2014} initialized the confusion matrices of all workers using the method of moments and showed convergence of the EM algorithm initialized with their method.
	However, they assumed that the minimum value of the entries of the confusion matrices is greater than a positive constant.
	\citet{Gao_et_al_2016} derived the minimax optimal convergence rate for the DS model, but they assumed that the class prior is uniform.
	
	\citet{Bonald_Combes_2017} gave a probabilistic concentration inequality for the estimation error between the ground truth and estimated truth without using the confusion matrices of workers. 
	While their result cannot be calculated empirically, our theoretical result can be calculated using the estimated confusion matrices and the class prior. 
	We describe the usefulness of calculating the bound in Section \ref{sec: ex_col}
	Because of our weaker assumptions, our theoretical analysis is applicable to all methods that use the estimation paradigm based on the ground truths of all tasks and confusion matrices of all workers.
	
\section{Dawid and Skene Model} \label{sec: mo}
	In this section, we formulate the problem of estimating the ground truths of tasks based on noisy and redundant labels provided by workers and their confusion matrices.
	The model formulated in this section is the basis of all models to be theoretically analyzed in the next section.

	Suppose we have $n$ tasks labeled by $m$ workers on $K$ possible labels.
	Let $X_{i, j}$ be the label of the $i$-th task given by the $j$-th worker.
	Denote by $X_{i, j}=k$ that the $j$-th worker labels $k \in \{1,2, \ldots,K\}$ to the $i$-th task.
	Denote by $X_{i,j} = 0$ that the $j$-th worker does not label the $i$-th task.
	We use $G_i \in \{ 1, 2, \ldots, K \}$ to denote the ground truth of the $i$-th task.
	Here, $\mathbb{S}^{K \times K}$ is the set of right stochastic matrices, that is, 
	\begin{eqnarray*}
		\forall P \in \mathbb{S}^{K \times K},  \forall k \in \{1, 2, \ldots, K\}, \sum_{k' = 1}^K P_{k, k'} = 1,
	\end{eqnarray*}
	\begin{eqnarray*}
		\forall P \in \mathbb{S}^{K \times K},  \forall k \in \{1, 2, \ldots, K\}, \forall k' \in \{ 1, 2, \ldots, K\}, P_{k, k'} \ge 0.
	\end{eqnarray*}
	The ability of the $j$-th worker is measured by a confusion matrix $\pi^j \in \mathbb{S}^{K \times K}$ with its $(k, k')$-element $\pi^j_{k, k'}$ being the probability that the $j$-th worker labels $k'$ when the true label is $k$. 
	For simplicity, we define 
	\begin{eqnarray*}
		X = \{ X_{i,j}\}_{i=1, j=1}^{n, m}, 
		G = \{ G_i \}_{i=1}^n,
		\pi = \{ \pi^j \}_{j = 1} ^m.
	\end{eqnarray*}
	
	We suppose that when $G_i = k$ is the ground truth of the $i$-th task, the label given by the $j$-th worker is sampled from a multinomial distribution  parametrized by the $k$-th row of the confusion matrix of the $j$-th worker, namely, $\pi^j_k$. 
	We assume that the ground truth of each task, $G_i$, is sampled from a multinomial distribution parametrized by $\rho=(\rho_1,\rho_2,\ldots,\rho_K)$, where $\rho_k \ge 0$ for any $k \in \{ 1, \ldots, K\}$ and $\sum_{k=1}^K \rho_k = 1$.
	We call $\rho$ the class prior of all tasks.
	We have observed variables $X$, latent variables $G$, and parameters $\{ \pi, \rho\}$.
	The graphical model is plotted in Figure \ref{graphical}-(a).
	
	Then, the joint distribution of this model is expressed as follows.
	\begin{eqnarray*}
		p(X, G | \rho, \pi) = \left[ \prod_{i=1}^n \prod_{j=1}^m p(X_{i.j} | G_i, \pi^j)\right] \left[ \prod_{i=1}^n p(G_i | \rho) \right],
	\end{eqnarray*}
	where each probability distribution is given as follows.
	\begin{eqnarray*}
		p(X_{i,j} | G_i, \pi^j) &=&\prod_{k'=1}^K \left( \pi^j_{G_i, k'} \right)^{\delta(X_{i,j} = k')},\\
		p(G_i | \rho) &=& \prod_{k=1}^K \rho_k ^ {\delta(G_i = k)}.
	\end{eqnarray*}
	
	The aim is to predict the ground truths of all tasks, $G$, from observed labels $X$.
	\citet{Dawid_Skene_1979} used the EM algorithm to infer the ground truths.
	However, when we conduct theoretical analysis in the next section, we take an approach that does not depend on a specific inference algorithm.
\begin{figure*}[th]
\vspace*{-6mm}
\begin{center}
$
\vspace*{-5mm}
\begin{array}{cc}
\vspace*{-2mm}
\includegraphics[width=40mm, angle=-90]{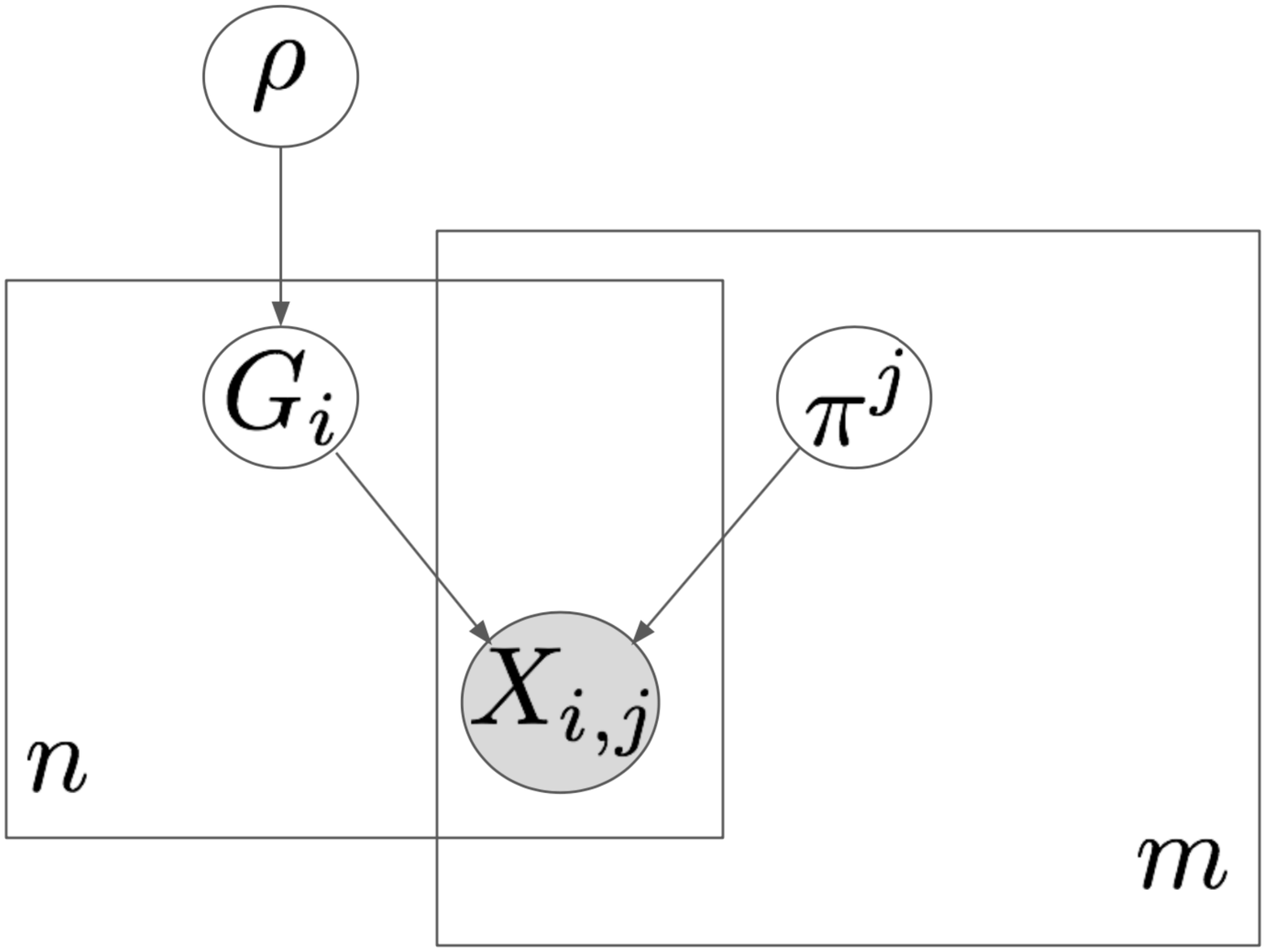}&
\includegraphics[width=40mm, angle=-90]{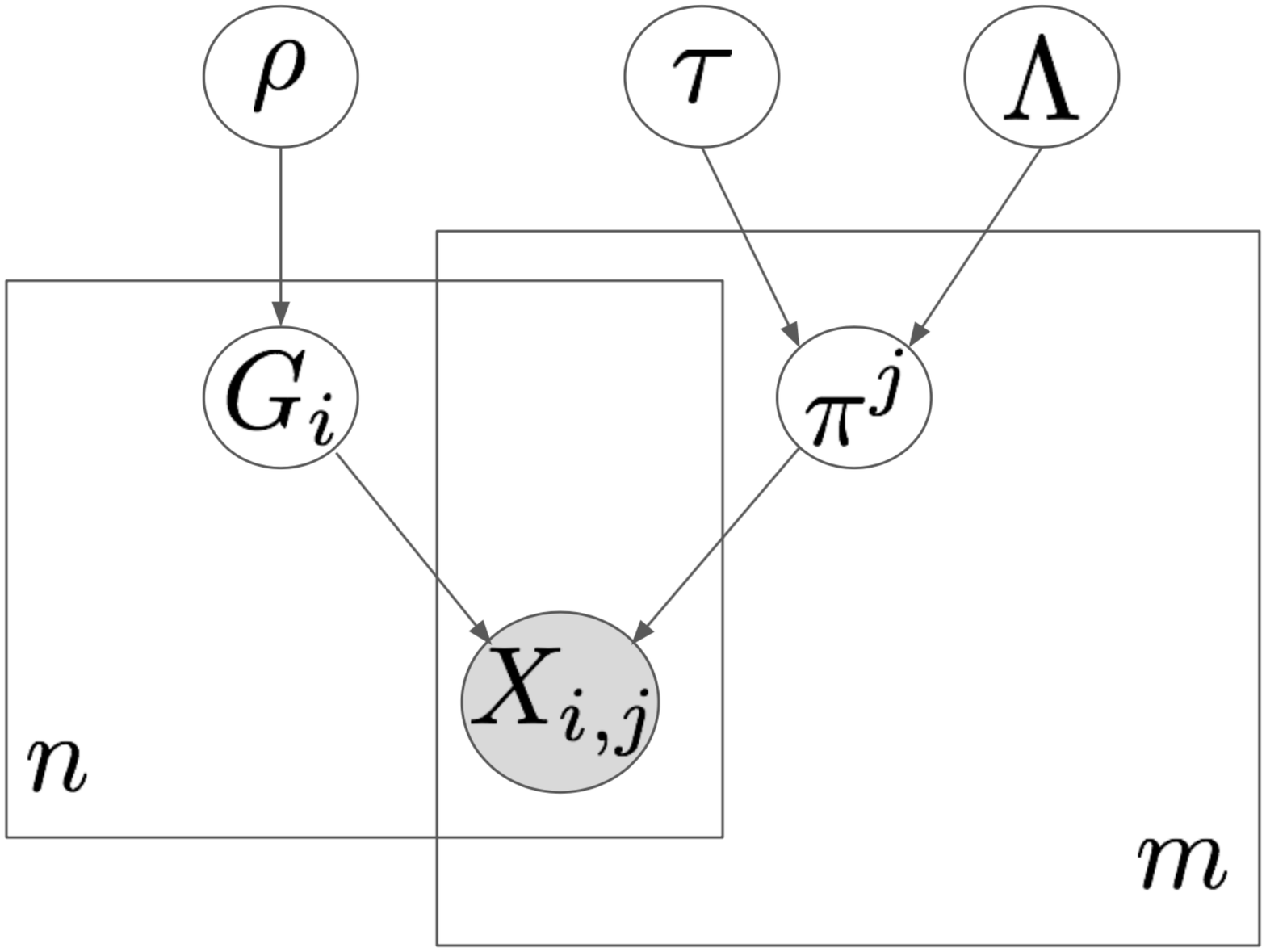}
\\
({\rm a})\:  {\rm The \: DS \: model. }&
({\rm b})\: {\rm The \: WC \: model. }
\end{array}
$
\end{center}
\caption{
Graphical models.
}
\label{graphical}
\end{figure*}

\section{Minimax Error Analysis} \label{sec: th}
	We give a lower bound on the minimax error for a class of models that use the ground truths of tasks and confusion matrices of workers.
	Let us define some concepts and formulas for analysis.
	
	We focus on the model described in Section \ref{sec: mo}.
	Note that the model may optionally contain additional structures, such as a prior for the class prior \cite{Liu_Wang_2012}.
	Moreover, we can use any inference algorithm to estimate the ground truths and class prior of tasks and the confusion matrices of workers.
	Let $\hat{G}$ be estimated truths of all tasks; let $\hat{G}_i$ be an estimate of the ground truth of the $i$-th task.
	The set of $\{ 1, \ldots, K \}$ is denoted by $[K]$.
	A loss is measured by the error rate given by
	\begin{eqnarray*}
		\mathscr{L}(\hat{G}, G) &=& \frac{1}{n} \sum_{i=1}^n \delta(\hat{G}_i \neq G_i),
	\end{eqnarray*}
	where $\delta(\cdot)$ is the indicator function.
	Let $\mathbb{P}$ be the joint probability distribution of the data $\{ X_{i,j} \}$ given $\pi, \rho$ and $G$.
	Let $\mathbb{E}$ be the associated expectation operator.
	Denote by $H(\rho)$ the entropy of the ground truth of each task with respect to the class prior:
	\begin{eqnarray*}
		H(\rho) &=& - \sum_{k=1}^K \rho_k \log \rho_k.
	\end{eqnarray*}
	Let $\mathrm{KL}(\pi^j_{g *} || \pi^j_{g' *} )$ be the Kullback-Leibler divergence from one row of the confusion matrix to the other:
	\begin{eqnarray*}
		\mathrm{KL}(\pi^j_{g *} || \pi^j_{g' *} ) &=&  \sum_{k=1}^K \pi^j_{g k} \log \frac{\pi^j_{g k}}{\pi^j_{g' k}}.
	\end{eqnarray*}
	
	We assume the following.
	\begin{ass} \label{ass: a}
		Given $\rho$ and $\pi$, the labels of tasks given by workers are in accordance with the models that use the ground truths of tasks and confusion matrices of workers such as the DS model.
	\end{ass}
	We bound the minimax error rate as follows.
	\begin{thm} \label{thm: bound}
	Under Assumption \ref{ass: a}, the minimax error rate is lower bounded as follows.
	\begin{eqnarray*}
	\inf_{\hat{G}} \sup_{G \in [K]^n} \mathbb{E} [\mathscr{L}(\hat{G}, G)] \ge
		\frac{1}{n \log K}\left( R(\rho, \pi) - \frac{\log 2}{n}\right),
	\end{eqnarray*}
	where
	\begin{eqnarray*}
		R(\rho, \pi) &=& H(\rho) - \sum_{j=1}^m \sum_{g=1}^K \sum_{g'=1}^K \rho_g \rho_{g'} \mathrm{KL}(\pi^j_{g *} || \pi^j_{g' *} ).
	\end{eqnarray*}
	\end{thm}
	
	The details of the proof are given in Appendix \ref{sec: pr}.
	The proof of this theorem is based on Fano's method by \citet{Yu_1997}, which is a well-known minimax lower bounding technique.
	
	Assumption \ref{ass: a} is weaker than that of many previous works such as \citet{Gao_et_al_2016} and \citet{Zhang_et_al_2014}.
	As we mentioned in the introduction, their approaches make more restrictive assumptions to conduct theoretical analysis.
	For example, \citet{Gao_et_al_2016} assumed that class prior $\rho$ is uniform, that is, $\rho_k = \frac{1}{K}$ for any $k \in \{ 1, \ldots, K \}$.
	 \citet{Zhang_et_al_2014} assumed that all entries $\pi^j_{g,k}$ of the confusion matrix is strictly positive.
	In contrast, thanks to the weak assumptions, our lower bound is applicable to the DS mode and other models that use the ground truths of tasks and confusion matrices of workers such as \citet{Zhang_et_al_2014} and \citet{Liu_Wang_2012}.

	Our lower bound can be used to measure the performance of each model.
	Specifically, the performance of each model can be measured by the value of $R(\rho, \pi)$, which is the main part of the lower bound of the minimax error rate.
	We conducted numerical experiments to measure the performance of each model,  the results  of which are given in the experiment section.
	
	In this paper, we provided the lower bound of the minimax error rate but did not give an upper bound.
	The theoretical analysis of the previous work such as \citet{Gao_et_al_2016}  easily derived an upper bound under mild assumptions.
	However, the derivation  depends not only on the model, but also on an inference algorithm, because $\inf_{\hat{G}} \sup_{G \in [K]^n} \mathbb{E} [\mathscr{L}(\hat{G}, G)]$ includes the infimum over estimate $\hat{G}$.
	Therefore, it is not appropriate to use such an upper bound when analyzing the behavior of the model itself.

\section{Worker Clustering Model} \label{sec: wc}
	We propose a model that is more practical than the DS model in the {\it laissez-faire crowdsourcing setting}, where workers can label as many tasks as they want.
	The main idea behind our model is simple: divide all workers into several disjoint clusters.
	Intuitively, by clustering workers, even if the number of labels provided by each worker is small, the number of labels per {\it worker cluster} can be increased to stabilize inference.
	
	The definition of the labels given by workers $X$, the ground truths of all tasks $G$, the class prior $\rho$, and the confusion matrices of all workers $\pi$ are the same as that of  the DS model. 
	The proposed model can actually be regarded as an extension of the DS model and the HybridConfusion model \cite{Liu_Wang_2012}.
	The biggest difference from the previous models is that the proposed model limits the possible values of confusion matrices $\pi$ to a maximum of $L$ values $\Lambda = \{ \Lambda_1, \Lambda_2, \ldots, \Lambda_L\}$, where $L \le m$.
	This means that we are clustering $m$ workers into a maximum of $L$ groups.	
	We also suppose that the confusion matrix $\pi^j$ is determined by a multinomial distribution parametrized by $\tau = (\tau_1, \tau_2, \ldots, \tau_L)$, that is, $\pi^j$ is equal to $\Lambda_l$ with probability $\tau_l$.
	This means that the $j$-th worker belongs to the $l$-th cluster with probability $\tau_l$. 
	We have observed variables $X$, latent variables $\{ G, \pi \}$ and parameters $\{ \Lambda, \rho,  \tau \}$.
	The graphical model is shown in Figure \ref{graphical}-(b).
	
	The joint distribution of this model is expressed as follows.
	\begin{align*}
		&p(X, G, \pi | \rho, \tau, \Lambda)\\
		=&\left[ \prod_{i=1}^n \prod_{j=1}^m p(X_{i.j} | G_i, \pi^j)\right] 
			\left[ \prod_{i=1}^n p(G_i | \rho) \right]
			\left[ \prod_{j=1}^m p(\pi^j | \tau, \Lambda)\right],
	\end{align*}
	where each probability distribution is given as follows.
	\begin{eqnarray*}
		p(X_{i,j} | G_i, \pi^j) &=& \prod_{k'=1}^K \left( \pi^j_{G_i, k'} \right)^{\delta(X_{i,j} = k')},\\
		p(G_i | \rho) &=& \prod_{k=1}^K \rho_k ^ {\delta(G_i = k)},\\
		p(\pi^j | \tau, \Lambda) &=& \prod_{l=1}^L \tau_l ^ {\delta(\pi^j = \Lambda_l)}.
	\end{eqnarray*}

	Note that our theoretical analysis given in Section \ref{sec: th} is also applicable to our proposed WC model.
	To apply the lower bound to our model, we use  $\Lambda$ instead of the confusion matrices as follows.
	We assume that $\Lambda$, class prior $\rho$, and a group $\ell_j $ which each worker belongs to  are known.
	Under this assumption, $\hat{\phi}_{j,l} = 1$ when $l = \ell_j$ and otherwise $\hat{\phi}_{j,l}=0$.
	Our lower bound of the minimax error rate is then established with $\Lambda_{\ell_j}$ instead of a confusion matrix $\pi^j$.
	
	To estimate latent variables and optimize parameters, we  adopt the strategy of \textit{empirical variational inference} \cite{Robbins_1956}.
	The details of the derivations are in Appendix B in the supplementary material.
	 
	The behavior of the WC model when $L$ gradually increases is as follows.
	Since $L$ is the maximum number of worker clusters, even if $L = m$, not all $m$ workers belong to different clusters.
	Some clusters may include many workers, but others may not include any workers.
	The actual number of clusters is adaptively determined for each target dataset.
	Therefore, it is not good to increase the value of $L$, but there should be an appropriate value of $L$.
	The method of determining the value of $L$ using our theoretical analysis given in Section \ref{sec: th} is described in Section \ref{sec: ex_col}.
	
\section{Experiments} \label{sec: ex}
	We empirically analyzed the proposed method on synthetic and real-world data.
	First, we analyzed that the WC model proposed in Section \ref{sec: wc} performs better than the existing method in the {\it laissez-faire crowdsourcing setting} where the number of labels per worker is in accordance with  Zipf's law.
	Second, we investigated that there is a strong similarity between the lower bound of the minimax error rate derived by our theoretical analysis in Section \ref{sec: th} and the empirical error of the estimated value.
	
	We compared our proposed method, the WC model, with the classical Majority Voting  (MV) scheme and the DS model.
	The experiments were conducted on four synthetic and three real-world datasets.
	
	\subsection{Synthetic Data} \label{sec: ex_syn}
		For synthetic data, we generated $m = 100$ workers and $n = 1000$ binary annotation tasks.
		The ground truth of each task was sampled from the binomial distribution with the class prior $\rho = (0.3, 0.7)$.
		For each worker, a confusion matrix was generated as follows.
		Workers were divided into two types: honest and adversarial.
		Each row of each confusion matrix was sampled from the Dirichlet distribution.
		The concentration parameter of the Dirichlet distribution for the $k$-th row was 
		\vspace*{-4mm}
		\begin{eqnarray*}
			\alpha = (\underbrace{10,\dots,10}_{k-1},100,\underbrace{10,\dots,10}_{K-k})
		\end{eqnarray*}
		for honest workers and
		\vspace*{-3mm}
		\begin{eqnarray*}
			\alpha = (\underbrace{10,\dots,10}_{k-1},1,\underbrace{10,\dots,10}_{K-k})
		\end{eqnarray*}
		for adversarial workers.
		To simulate the labeling process of a crowd, we consider two settings:		
		The first setting is the ``Constant Labeling Number (CLN)'' setting , in which we determine the number $N$ of tasks that all workers commonly annotate, and then each worker is given $N$ tasks randomly and label those tasks in acordance with his or her own confusion matrix.
		The second setting is the ``Zipf Labeling Number (ZLN)'' setting, in which, for any $j \in \{1, \ldots, m\}$, we determines the number $N_j$ of tasks that the $j$-th worker  annotates in accordance with Zipf's law, and  then the $j$-th worker is randomly given $N_j$ tasks and labels those tasks according to his or her own confusion matrix. 
		The CLN setting was assumed in the experiments of many existing studies \citep[e.g.][]{Welinder_et_al_2010, Snow_et_al_2008}.
		The ZLN setting is suitable for the {\it laissez-faire crowdsourcing setting}.
		
		The input of each model was a set of triplets (task, worker, and label) generated by this procedure.
		The outputs of the WC and DS models were estimated values of the ground truths and class prior of all tasks and the confusion matrices of all workers.
		In order to measure how robust each method is to the proportion of adversaries, we showed the change in accuracy as a function of the proportion of adversaries (see Figure \ref{results}).
		We changed the proportion of adversaries from $0$ to $50$\%.
		
		The accuracy of MV monotonically decreases as the proportion of adversaries increases in all cases.
		In the CLN setting, when the number of times each worker labels is sufficiently large ($N = 100$), the accuracies of DS and WC remain high even if the proportion of adversaries is high (nearly $0.5$) (see Figure \ref{results}-(c)).
		In contrast, when the number of times each worker labels is small ($N = 2, 10$),   the accuracy of DS is low (almost the chance level), but that of WC remains high  (see Figure \ref{results}-(a), (b)).
		Intuitively, the reason for this phenomenon is that  WC effectively increased the number of labeled tasks per cluster by clustering workers and thus the estimation became more reliable.
		In particular, this phenomenon is noticeable when the number of tasks to be labeled per worker follows Zipf's law, i.e., in the ZLN setting (see Figure \ref{results}-(d)).
		Even if the average or the maximum value of the number of tasks to be labeled per worker is large, when most workers only label a small number, the accuracy of DS is not high.
		When the number of times most workers label is small, the proposed method WC is more effective than MV and DS.

	\subsection{Real-World Data} \label{sec: ex_re}
		In the real-world data experiments, we compared crowdsourcing algorithms on three datasets: two binary tasks and one multi-class task.
		The two binary tasks were labeling bird species \cite{Welinder_et_al_2010} (the Bird dataset) and recognizing textual entailment \cite{Snow_et_al_2008} (the RTE dataset).
		The multi-class task was labeling the breed of dogs \cite{Zhou_et_al_2012} (the Dog dataset).
		The statistics for the datasets were summarized in Table \ref{ta: st}.
		
		\begin{table}[t]
			\caption{Statistics of real-world datasets.}
				\label{ta: st}
				\vspace*{-3mm}
				\begin{center}
				\begin{small}
				\begin{sc}
					\begin{tabular}{lcccr}
						\toprule
						Dataset & $\sharp$classes  & $\sharp$tasks  & $\sharp$workers & $\sharp$labels  \\
						\midrule
						Bird & 2 & 108 & 39 & 4212 \\
						RTE & 2 & 800 & 164 & 8000\\
						Dog & 4 & 807 & 52 & 7354  \\
						\bottomrule
					\end{tabular}
				\end{sc}
				\end{small}
				\end{center}
				\vskip -0.3in
		\end{table}
				\begin{table}[t]
			\caption{Error rate (\%) of estimating the ground truths of tasks on real-world datasets. MV, DS, and WC stand for majority voting, Dawid and Skene, and worker clustering, respectively.}
				\label{ta: ac}
				\begin{center}
				\begin{small}
				\begin{sc}
					\begin{tabular}{lccr}
						\toprule
						Model \textbackslash Dataset & RTE & Bird & Dog \\
						\midrule
						MV  & 8.12 & 24.07 & 18.46\\
						DS & 7.62 & 15.74 & 16.60 \\
						WC ($L = 1$) &  11.37 & 19.44 & 18.71\\
						WC ($L = 2$) & 7.00 & 17.59& 18.59\\
						WC ($L = 10$) & {\bf 6.75} & {\bf 11.11} & 16.98\\
						WC ($min R(\rho, \pi)$) & {\bf 6.75} &{\bf 11.11} & {\bf 16.36} \\
						\bottomrule
					\end{tabular}
				\end{sc}
				\end{small}
				\end{center}
				\vskip -0.3in
		\end{table}
		
		We calculated the error rate of the WC model by changing the value of parameter $L$. 
		In Table \ref{ta: ac}, we reported the values for $L = 1, 2, 10$ and the value of the error rate when $R(\rho, \pi)$ attains its minimum.
		Table \ref{ta: ac} also shows the error rate calculated by MV and DS as baselines.
		It can be seen that the error rate decreases as the value of $L$ increases in the WC model.
		For the RTE dataset, the WC model is superior to both MV and DS when $L \ge 2$.
		For the Bird dataset, the WC model outperforms MV for all $L$ and DS when $L\ge10$.
		For the Dog dataset, the WC model does not outperform DS even when $L = 10$; as the value of $L$ is increased, the WC model outperforms DS when $R(\rho, \pi)$ is minimum.
		
		As expected, the performance of the WC model depends heavily on the choice of $L$.
		However, the experimental results show that larger $L$ always yields better performance. 
		We will investigate this issue more systematically in Section \ref{sec: ex}.3.
			
	\subsection{Similarity between Lower Bound of Minimax Error Rate and Empirical Error} \label{sec: ex_col}
		We experimentally investigated the similarity between the lower bound of the minimax error rate and the empirical error of the ground truths of tasks estimated by the WC model.
		
		It is important to note that the lower bound of the minimax error rate derived by our theoretical analysis can be approximately calculated from the class prior and confusion matrices estimated by each method, but the empirical error of the ground truths of tasks estimated by each method can not be calculated without the ground truths of those tasks.
		In practice, we want to adopt a model that makes the empirical error as small as possible, but its calculation is impossible because of the necessity of the ground truths of tasks.
		If there is some similarity between  the empirical error  and lower bound of the minimax error rate, we can estimate the behavior of each model by the theoretical lower bound even in a more realistic situation in which we do not know the ground truths of tasks.
		
		Here we experimentally investigate the behavior of $R(\rho, \pi)$ and $\mathscr{L}(\hat{G}, G)$ against the change in the value of the maximum number of worker clusters, $L$ for the WC model.
		 $R(\rho, \pi)$ is the main part of the lower bound of the minimax error rate derived by our theoretical analysis introduced in Section \ref{sec: th}, while $\mathscr{L}(\hat{G}, G)$ is the empirical error of the estimated truths of tasks by the WC  model.
		
		The experiment was conducted on three synthetic datasets and three real-world datasets.
		For each dataset, we plotted in Figure \ref{hoge} the change in $R(\rho, \pi)$ and $\mathscr{L}(\hat{G}, G)$ as a function of the parameter $L$.
		The value of $L$ was changed from $1$ to $m$.
		In Figure \ref{hoge}, the results of $\mathscr{L}(\hat G, G)$ and $R(\rho,\pi)$ are shown in the left and right columns, respectively. 
		
		For all cases, we see that increasing the value of $L$ tends to decrease the values of $ R (\rho, \pi) $ and $ \mathscr{L}(\hat{G}, G) $ similarly. 
		This means that the behavior of $\mathscr{L}(\hat{G},G)$ can be well predicted by $R(\rho,\pi)$.
		Looking at the change in $\mathscr{L}(\hat{G}, G)$, to achieve better performance, we can see that it is sufficient to set the value of $L$ to be large to some extent.
		This behavior can be well captured by looking at the change in $R(\rho, \pi)$ which can be computed in practice. 
		Thus, our theoretical lower bound may be useful to investigate the behavior of different models.
		
\section{Conclusions}
	In this paper, we gave a novel theoretical error analysis by using Fano's method for any models based on the ground truths of tasks and the confusion matrices of workers.
	We have considered a realistic {\it laissez-faire crowdsourcing setting} and proposed clustering workers to reduce the effect of the bias in the number of tasks to be labeled per worker.
	Our theoretical analysis is applicable to all models that use the ground truths of tasks and the confusion matrices of workers, including \citet{Dawid_Skene_1979} and its variations, thanks to the weak assumptions required in our analysis.
	Through experiments on synthetic and real-world data, we have found that there is a strong similarity between the lower bound of the minimax error rate derived by our theoretical analysis and the empirical error of the estimated value.
	Thus, our theoretical lower bound, which can be approximately calculated from an established class prior and estimated confusion matrices, can be useful to investigate the behavior of different models in practice.
\paragraph*{Acknowledgement}
IS was supported by JST CREST Grant Number JPMJCR17A1, Japan.
and MS was supported by KAKENHI 17H00757. 

\appendix
\section{Proof of Theorem \ref{thm: bound}} \label{sec: pr}

Fano's method \citep{Yu_1997} uses Fano's inequality \citep{Fano_1949} to give a lower bound for the minimax error rate.
In the method, the uniformity of a class prior is assumed.
However, it is not realistic for crowdsourcing.
Therefore, we do not use Fano's method directly but improve the method to be suitable for crowdsourcing.
We use the following inequality between the error probability and a conditional entropy.
\begin{lem} [Fano's inequality \citet{Fano_1949}] \label{lem: fanoo}
	\ \\
	For any Markov chain $V \rightarrow X \rightarrow \hat{V}$, we have
	\begin{eqnarray*}
		h (P(\hat{V} \neq V)) + P(\hat{V} \neq V) (\log (|\mathscr{V}| - 1) \ge H(V | \hat{V}),
	\end{eqnarray*}
	where $h(p) = - p \log p - (1 - p) \log (1 - p)$, 
	$\mathscr{V}$ is the set of possible value of $V$, and $H(V | \hat{V})$ is the entropy of $V$ conditioned on $\hat{V}$.
\end{lem}

The details of the proof are given in Appendix A in supplementary material.
Using this inequality, we prove Theorem \ref{thm: bound} as follows.
\vspace*{-3mm}
\begin{proof}[Proof of Theorem \ref{thm: bound}]
	\ \\
	First, using Markov's inequality, we have
	\vspace*{-3mm}
	\begin{align*}
		\inf_{\hat{G}} \sup_{G \in [K]^n} \mathbb{E} [\mathscr{L}(\hat{G}, G)] 
		&\ge \inf_{\hat{G}} \sup_{G \in [K]^n} \frac{1}{n}\mathbb{P} \left[ \left( \mathscr{L}(\hat{G}, G) \ge \frac{1}{n} \right) \right] \\
		&= \inf_{\hat{G}} \left[ \frac{1}{n} \sup_{G \in [K]^n} \mathbb{P} \left(  \mathscr{L}(\hat{G}, G) \ge \frac{1}{n} \right)\right] \\
	\end{align*}
	\vspace*{-1mm}
	Since $\hat{G} \neq G$ and $\mathscr{L}(\hat{G}, G) = \frac{1}{n} \sum_{i=1}^n \delta(\hat{G}_i \neq G_i) \ge \frac{1}{n}$ are equivalent, we have
	\vspace*{-3mm}
	\begin{eqnarray}
		\inf_{\hat{G}} \sup_{G \in [K]^n} \mathbb{E} [\mathscr{L}(\hat{G}, G)] 
		&\ge& \inf_{\hat{G}} \left[ \frac{1}{n}  \sup_{G \in [K]^n} \mathbb{P} \left(  \mathscr{L}(\hat{G}, G) \ge \frac{1}{n} \right)\right] \nonumber \\
		&=& \inf_{\hat{G}} \left[ \frac{1}{n}  \sup_{G \in [K]^n} \mathbb{P} \left(  \hat{G} \neq G\right)\right]. \label{preq1} 
	\end{eqnarray}
	
	Second, we evaluate $\mathbb{P} \left(  \hat{G} \neq G\right)$ by using Lemma (\ref{lem: fanoo}), i.e., 
	\vspace*{-3mm}
	\begin{align}
		 \mathbb{P}\left(  \hat{G} \neq G\right) \nonumber 
		&\ge \mathbb{P}\left(  \hat{G} \neq G\right) \frac{\log (K^n - 1)}{\log K^n}\nonumber \\
		&\ge   \frac{\mathbb{P}\left(  \hat{G} \neq G\right) \log (K^n - 1) + h (P(\hat{G} \neq G)) - \log 2}{\log K^n}\nonumber\\
		&\ge \frac{H(G | \hat{G}) - \log 2}{n \log K}. \label{preq2}
	\end{align}
	Third, we evaluate $H(G | \hat{G})$ by using the relationship between conditional entropy and mutual information, and the data processing inequality, i.e., 
	\begin{eqnarray*}
		H(G | \hat{G}) &=& H(G) - I(G ; \hat{G}) \\
		&\ge& H(G) - I(G ; X).
	\end{eqnarray*}
	Since the elements of $G = (G_1, \ldots, G_n)$ are independent and identically distributed, we have
	\begin{eqnarray}
		H(G | \hat{G}) \ge   n H(\rho) - I(G ; X) \label{preq3}
	\end{eqnarray}
	Then the mutual information of $G$ and $X$ is evaluated as follows.
	\begin{align*}
		&\:I(G; X) \nonumber \\
		&= KL(P(G, X) \left\|  P(G)P(X)) \right.\\
		&= KL(P(X| G)P(G) || P(G)P(X))\\
		&= \sum_{G \in [K]^n, X \in [K]^{mn}} P(G)P(X|G) \log \frac{P(G)P(X|G)}{P(G)P(X)}\\
		&= \sum_{G \in [K]^n} P(G) KL(P(X | G) || P(X)) \\
		&= \sum_{G \in [K]^n} P(G) KL\left( P(X | G) \left\| \sum_{G' \in [K]^n} P(X | G') P(G') \right.\right)  \\
		&\le\sum_{G \in [K]^n} P(G) \sum_{G' \in [K]^n} P(G') KL\left( P(X | G) || P(X | G') \right) \\
		&\: ( \mbox{by convexity of} - \log ) \\
		&= \sum_{G \in [K]^n} \sum_{G' \in [K]^n} P(G) P(G') \left( \sum_{i=1}^n \sum_{j=1}^m KL\left( P(X_{i,j} | G) || P(X_{i,j} | G') \right) \right) \\
		&\: ( \mbox{by independence of } X ) \\
		&= \sum_{G \in [K]^n} \sum_{G' \in [K]^n} \sum_{i=1}^n \sum_{j=1}^m P(G) P(G') KL(\pi^j_{G_i *} || \pi^j_{G'_i *} )  \\
		&= \sum_{i=1}^n \sum_{j=1}^m \sum_{g = 1}^K \sum_{g' = 1}^K P(G_i = g) P(G'_i = g') KL(\pi^j_{g *} || \pi^j_{g' *} )  \\
		&= n \sum_{j=1}^m \sum_{g = 1}^K \sum_{g' = 1}^K \rho_g \rho_{g'} KL(\pi^j_{g *} || \pi^j_{g' *} ).  \\
	\end{align*}
	Combining this inequality and Inequality (\ref{preq3}), we have
	\begin{eqnarray}
		H(G | \hat{G}) &\ge& n \left(H(\rho) -  \sum_{j=1}^m \sum_{g = 1}^K \sum_{g' = 1}^K \rho_g \rho_{g'} KL(\pi^j_{g *} || \pi^j_{g' *} ) \right) \nonumber \\ 
		&=& n R(\rho, \pi) \label{eq: two}. \label{preq4}
	\end{eqnarray}
	
	Finally, using Inequality (\ref{preq1}), (\ref{preq2}), and (\ref{preq4}), we get
	\begin{align*}
		\inf_{\hat{G}} \sup_{G \in [K]^n} \mathbb{E} [\mathscr{L}(\hat{G}, G)] 
		&\ge \inf_{\hat{G}} \left[ \frac{1}{n} \sup_{G \in [K]^n}  \frac{1}{\log K} \left( R(\rho, \pi) - \frac{\log 2}{n}\right)\right] \\
		&= \frac{1}{n \log K} \left( R(\rho, \pi) - \frac{\log 2}{n}\right)
	\end{align*}
	This completes the proof.
\end{proof}

	\begin{figure*}[t]
\begin{center}
$
\begin{array}{cccc}
\includegraphics[height = 0.4\columnwidth]{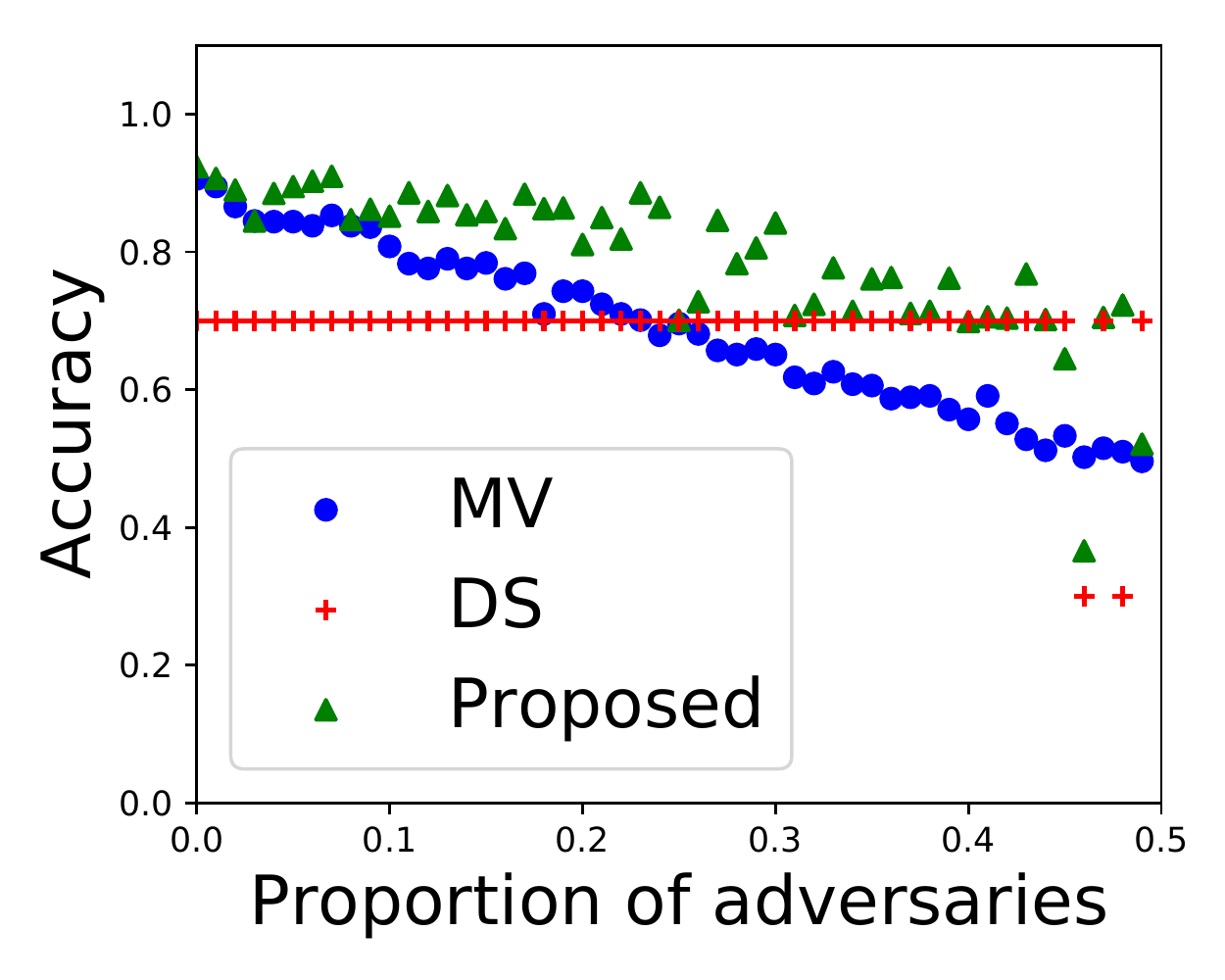}&
\includegraphics[height = 0.4\columnwidth]{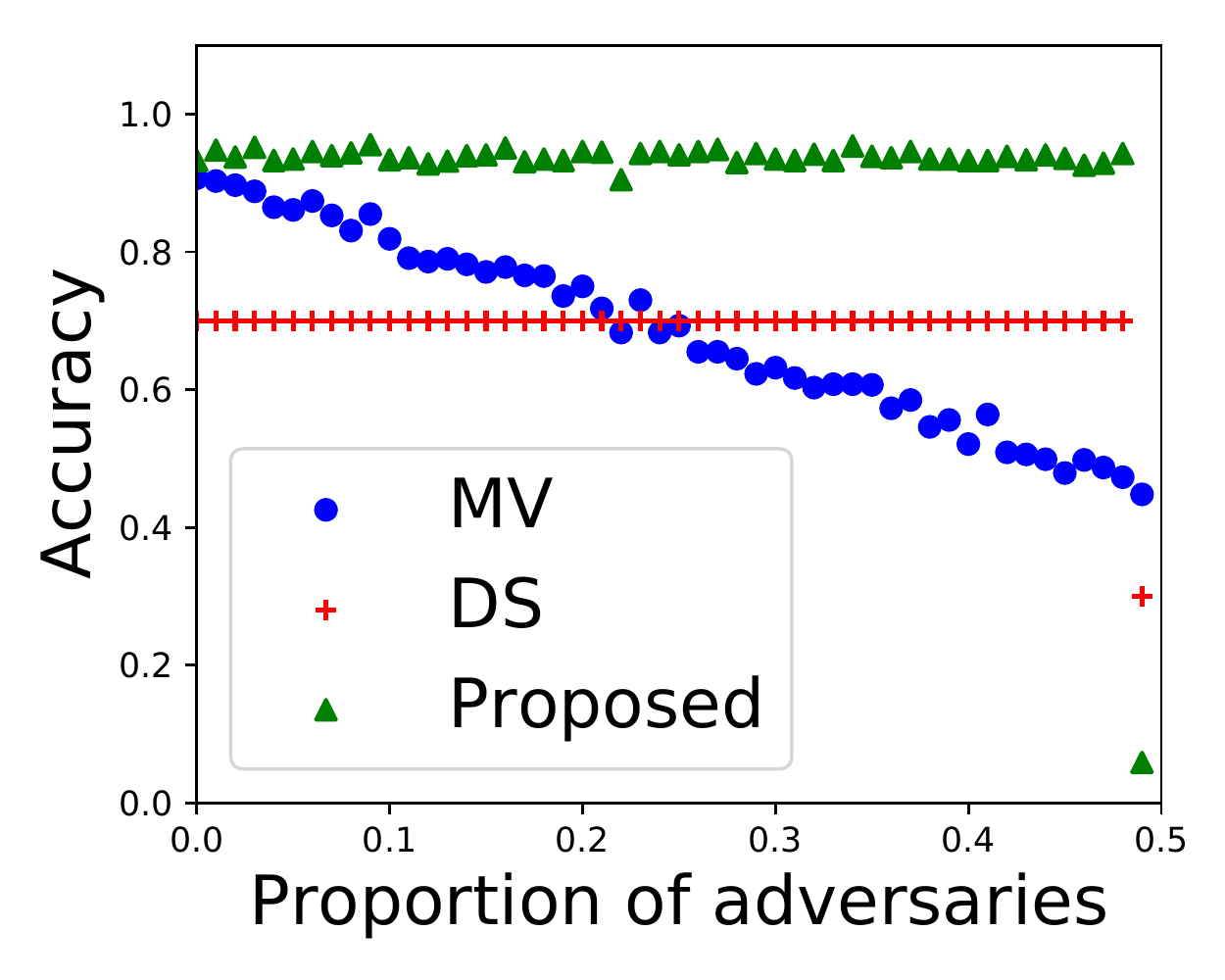}&
\includegraphics[height = 0.4\columnwidth]{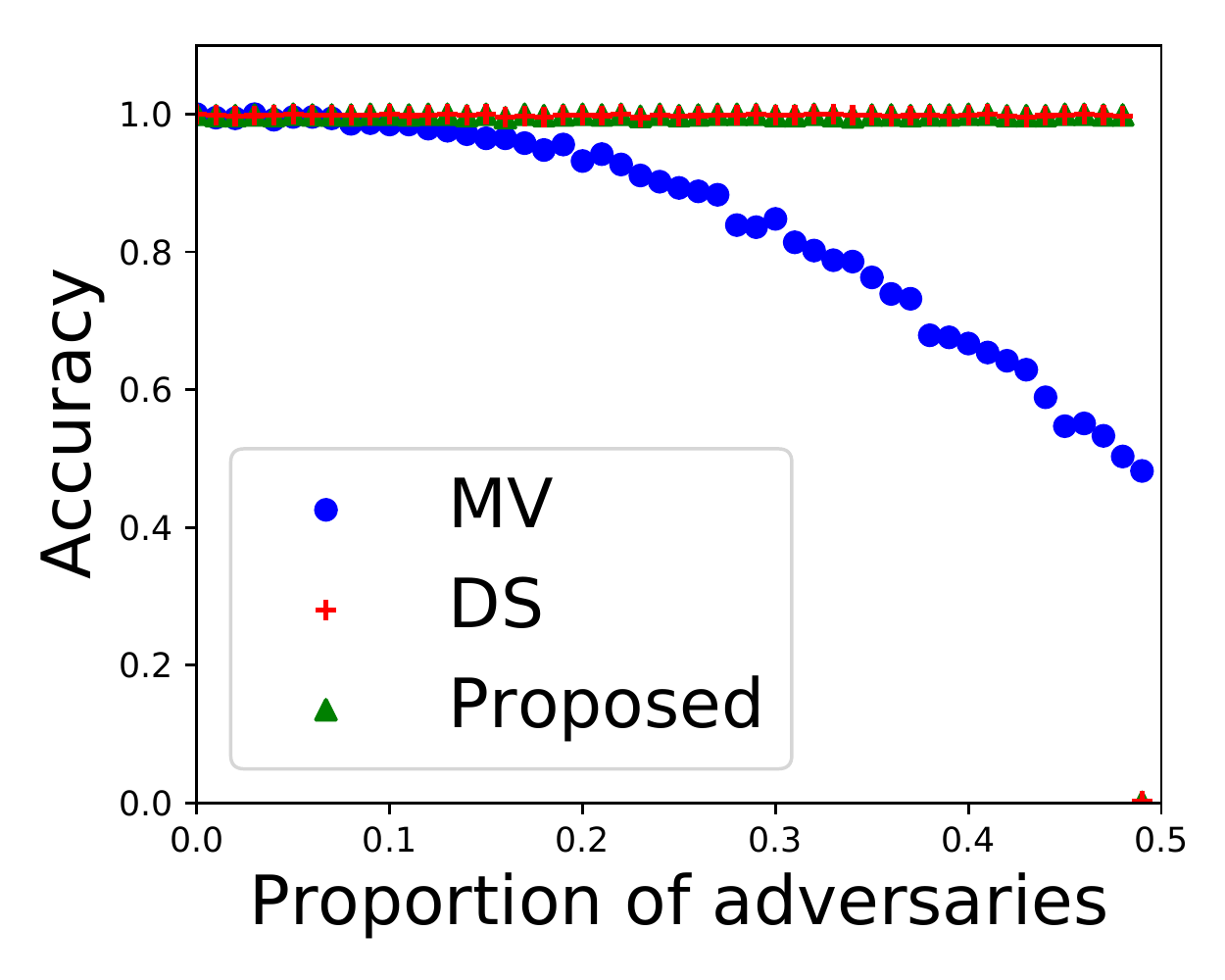}&
\includegraphics[height = 0.4\columnwidth]{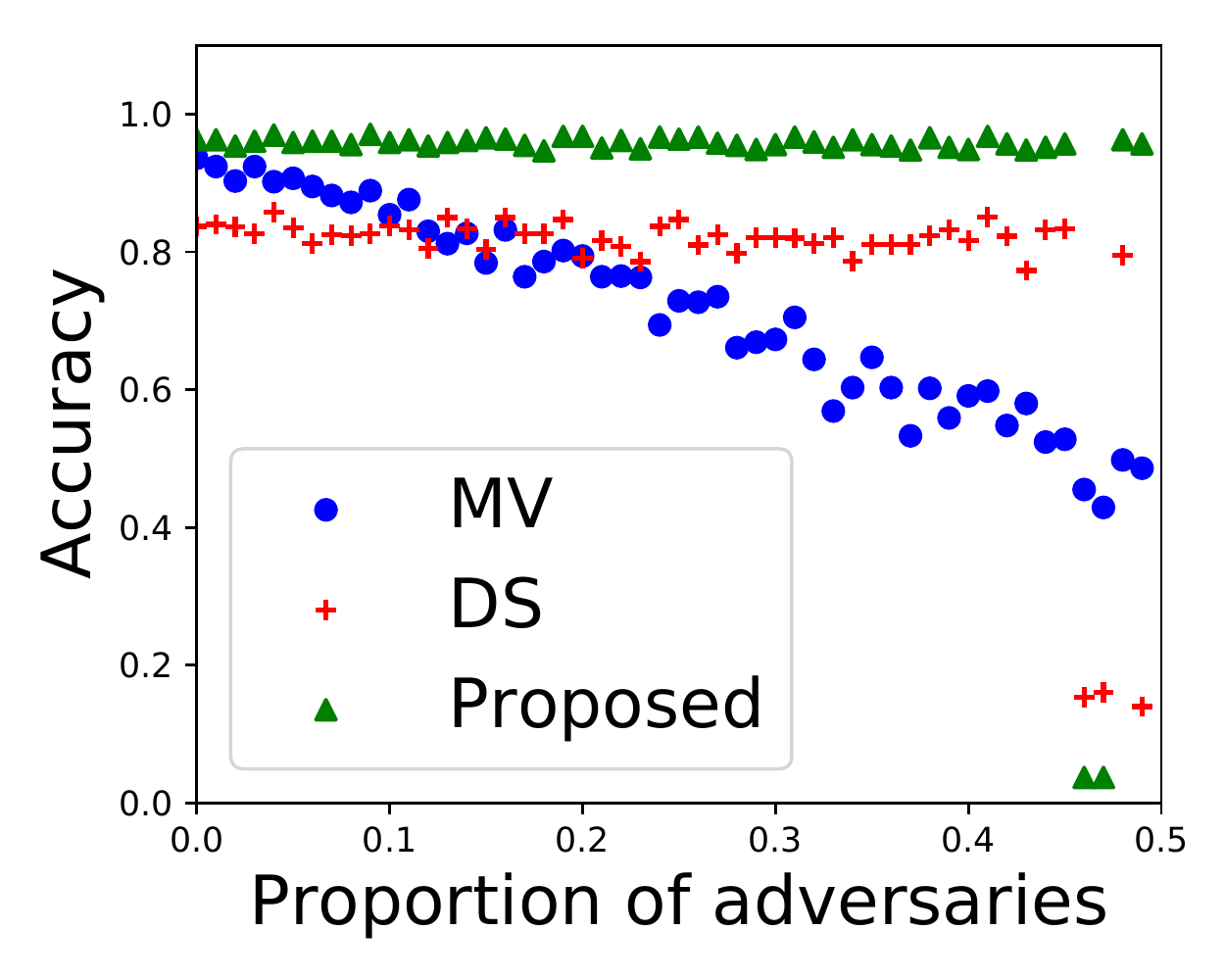}
\\
({\rm a}) {\rm CLN }\: (N = 2)
&
({\rm b}) {\rm CLN }\: (N = 10)
&
({\rm c}) {\rm CLN }\: (N = 100)
&
({\rm d}) {\rm ZLN \:}  (N = 1, \ldots, 60)
\end{array}
$
\end{center}
\caption{
Change of accuracy as a function of the proportion of adversaries.
MV, DS, and WC stand for majority voting, Dawid and Skene, and worker clustering, respectively.
}
\label{results}
\end{figure*}

\begin{figure*}[t]
\begin{center}
$
\begin{array}{cc}
\includegraphics[width = \columnwidth]{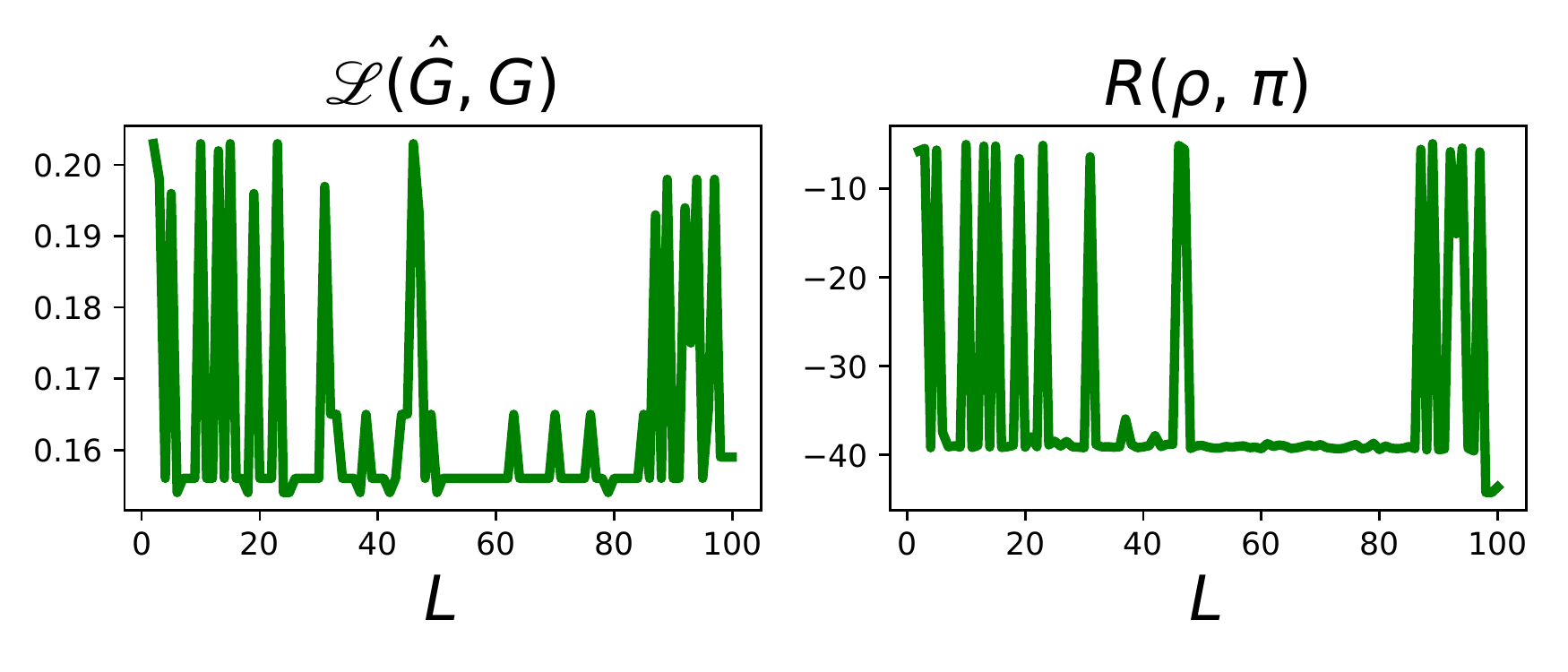}&
\includegraphics[width = \columnwidth]{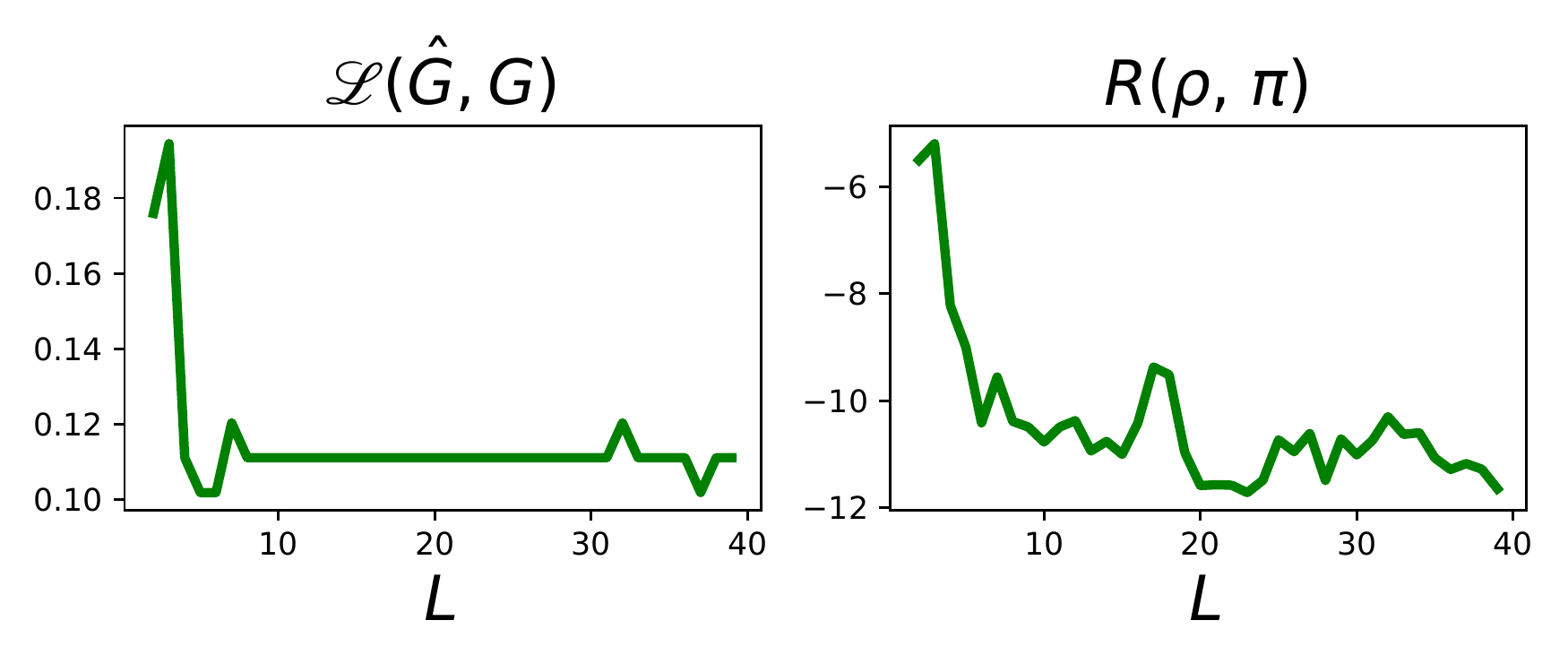}
\\
({\rm a})\: {\rm CLN }\: (N = 2)
&
({\rm d})\: {\rm Bird \: dataset }
\\ \
\\ \
\\
\includegraphics[width = \columnwidth]{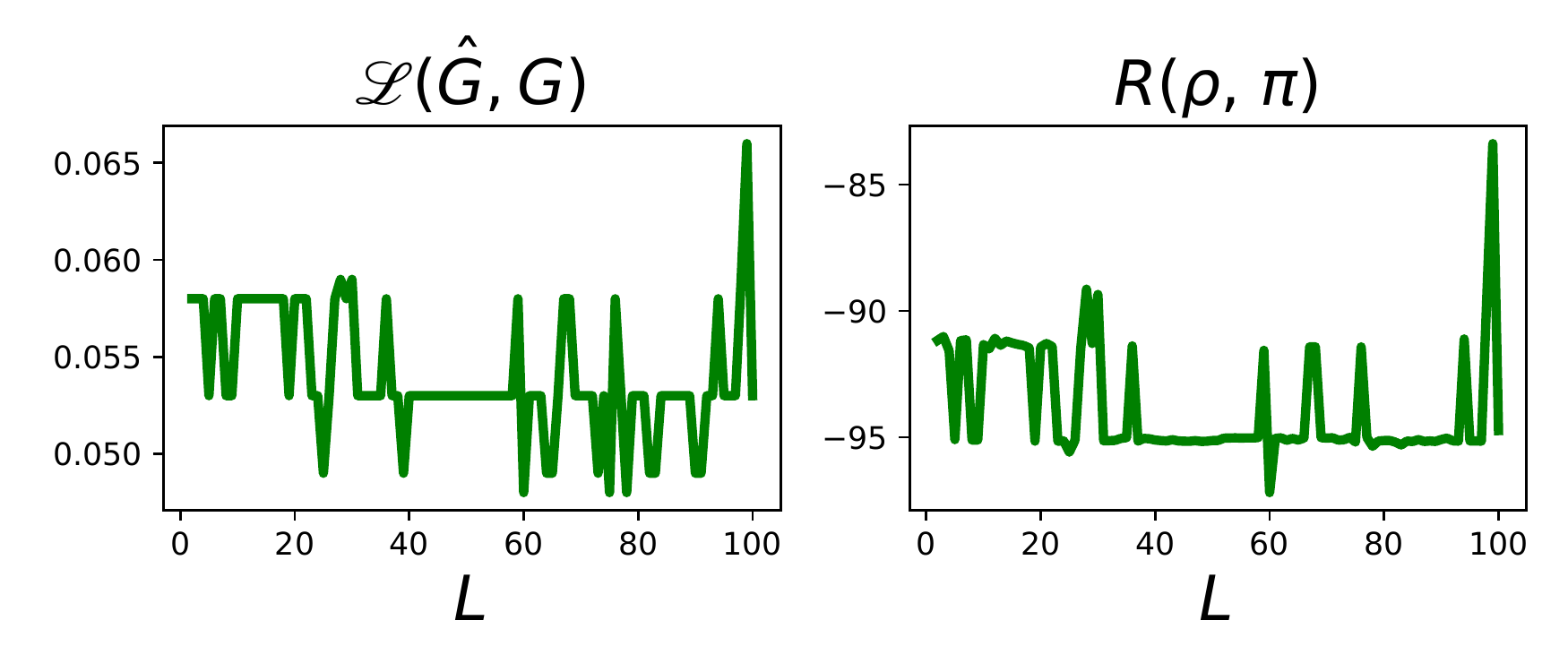}&
\includegraphics[width = \columnwidth]{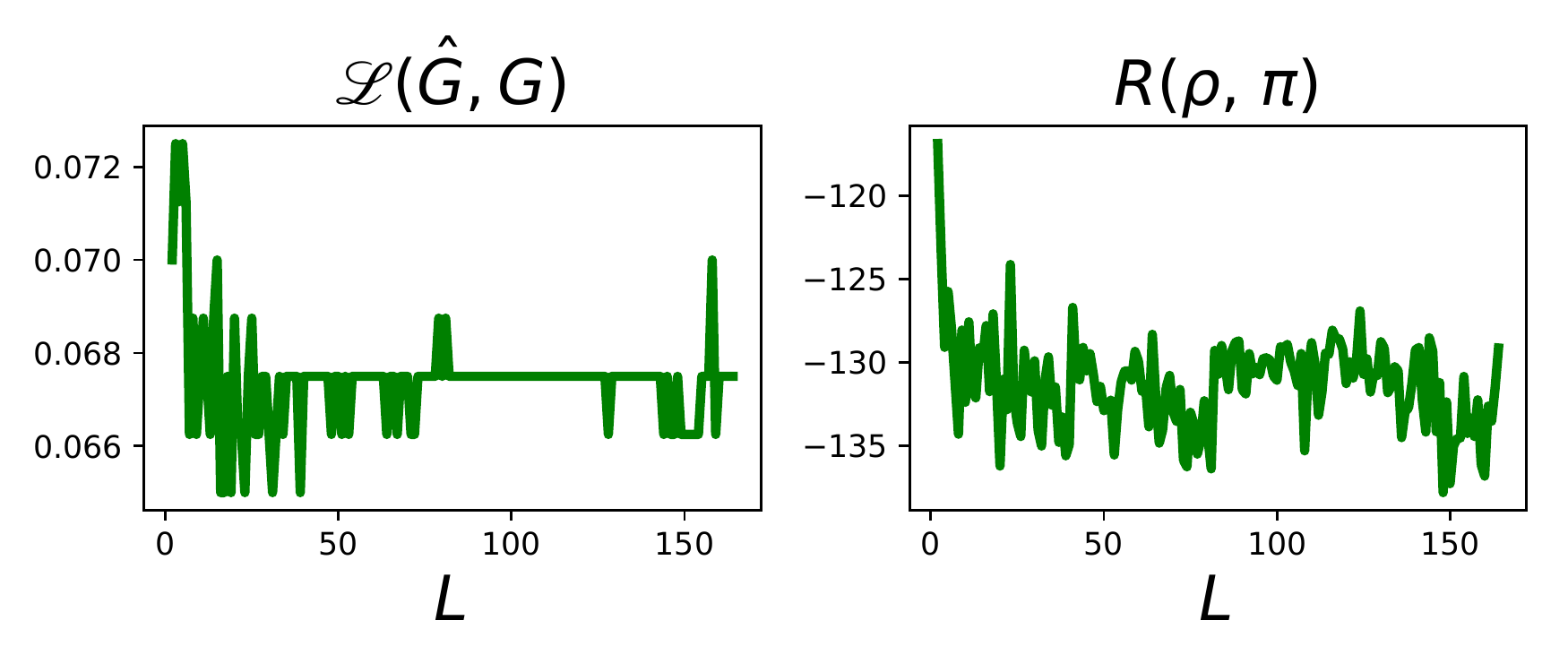}
\\
({\rm b})\: {\rm CLN }\: (N = 10)
&
({\rm e})\: {\rm RTE \: dataset }
\\ \
\\ \
\\
\includegraphics[width = \columnwidth]{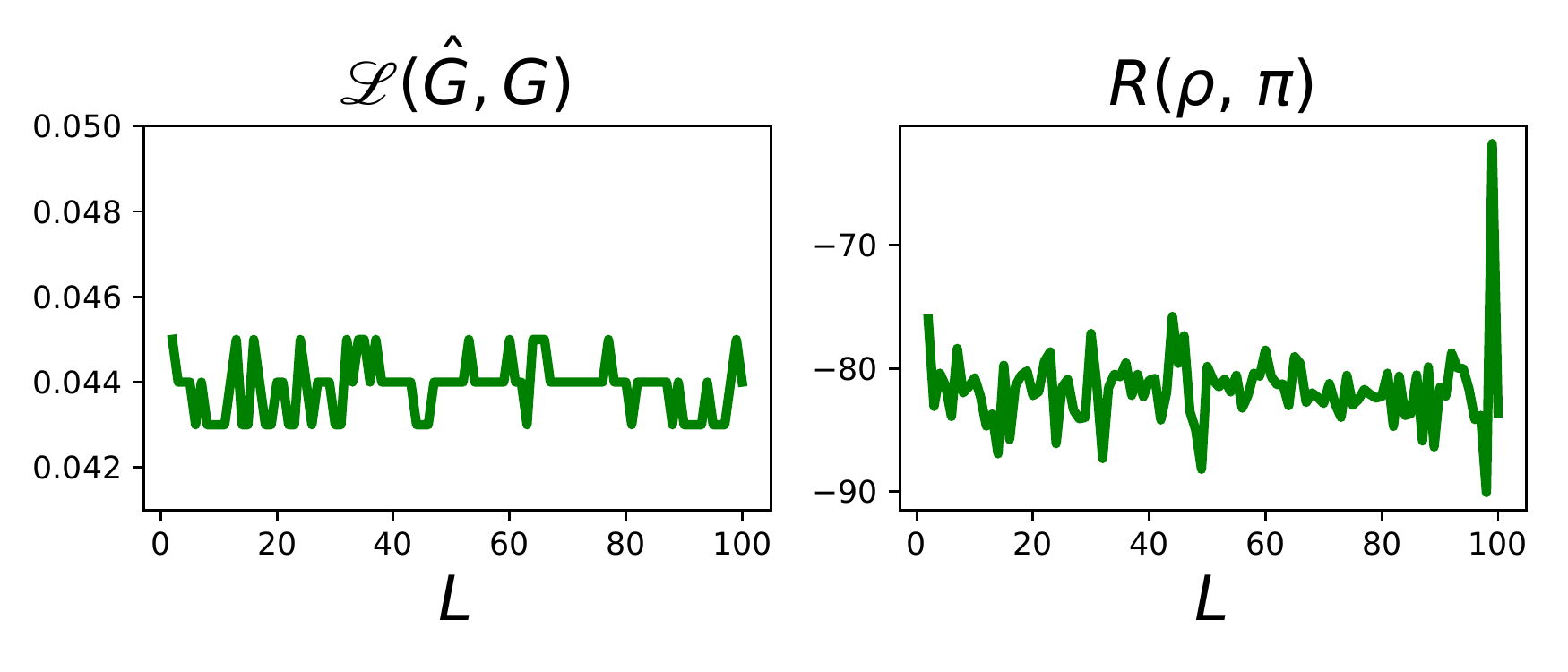}&
\includegraphics[width = \columnwidth]{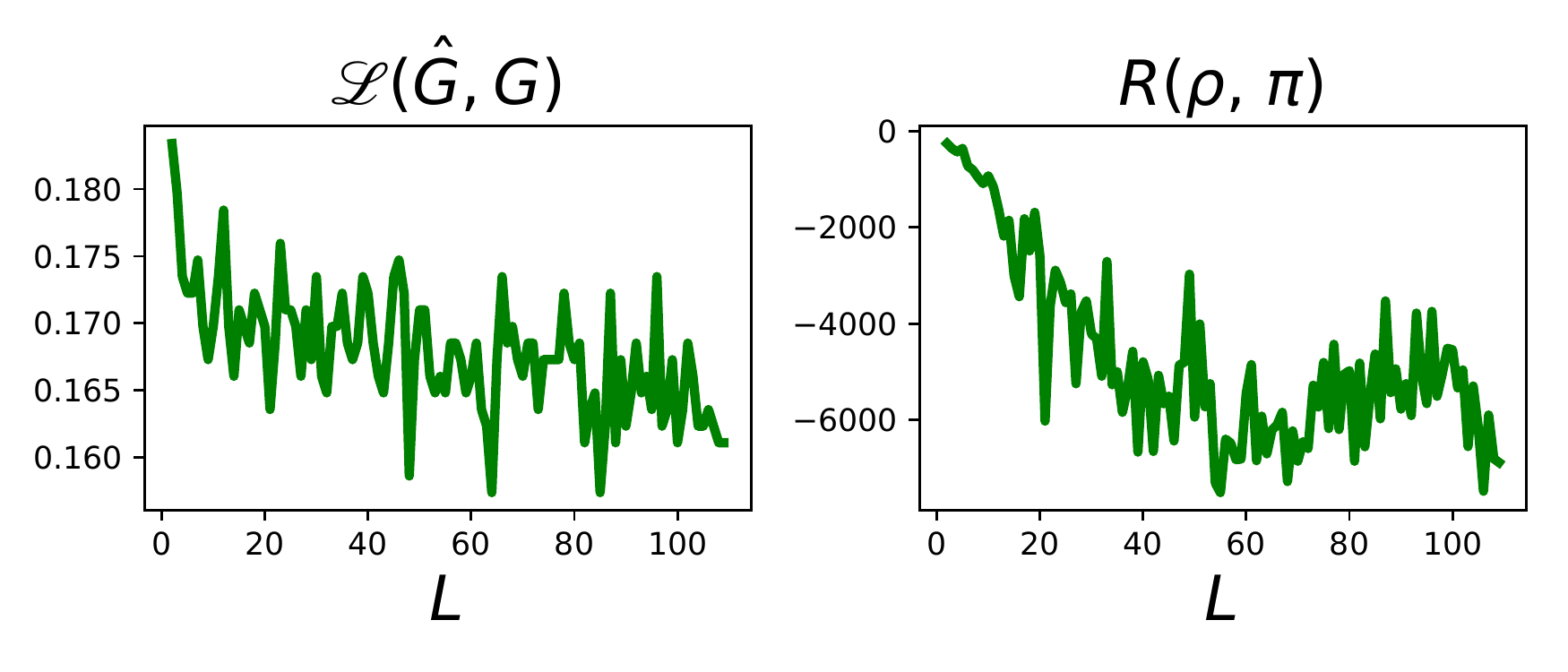}
\\
({\rm c})\: {\rm ZLN }\: (Z = 1, \ldots, 60)
&
({\rm f})\: {\rm Dog \: dataset }
\end{array}
$
\end{center}
\caption{
Change in the empirical error $\mathscr{L}(\hat{G}, G)$ and the main part of the lower bound of the minimax error rate $R(\rho, \pi)$ for the maximum number of worker clusters $L$.
}
\label{hoge}
\end{figure*}

\clearpage

\bibliography{ref}
\bibliographystyle{icml2018}

\end{document}